%% file: main.tex
\documentclass{article}

\usepackage{fullpage}
\usepackage{amsmath,amssymb,amsthm,mathtools}
\usepackage{booktabs,tabularx,array}
\usepackage{graphicx}
\usepackage{flafter}
\usepackage{float}
\usepackage{placeins}
\usepackage{microtype}
\usepackage{xcolor}
\usepackage{url}
\usepackage{natbib}
\usepackage{hyperref}

\setcitestyle{authoryear,round,citesep={;},aysep={,},yysep={;}}
\hypersetup{
  colorlinks=true,
  citecolor=blue,
  linkcolor=blue,
  urlcolor=blue,
  pdfauthor={Vicente Opazo},
  pdftitle={When a Flatness Proxy Is Not a Function: Robustness Certificates and Training Interventions}
}

\newtheorem{proposition}{Proposition}
\newtheorem{corollary}{Corollary}
\newcommand{\R}{\mathbb{R}}
\newcommand{\one}{\mathbf{1}}
\newcommand{\Braw}{B_{\mathrm{raw}}}
\newcommand{\Bquot}{B_{\mathrm{quot}}}
\newcommand{\kexact}{\kappa_{\mathrm{exact}}}
\newcommand{\softmax}{\operatorname{softmax}}
\newcommand{\Tr}{\operatorname{Tr}}

\title{When a Flatness Proxy Is Not a Function:\\
Robustness Certificates and Training Interventions}

\author{
  Vicente Opazo\\
  \small Centro Nacional de Inteligencia Artificial (CENIA)\\
  \small\texttt{vicente.opazo@cenia.cl}
}
\date{}

\begin{document}

\maketitle

\begin{abstract}
    A valid curvature upper bound need not justify either a robustness certificate or an intervention on an intrinsic predictor property. We demonstrate this distinction for a last-layer relative-flatness proxy used in both settings. First, empirical-risk stationarity does not eliminate pointwise first-order loss terms: at a finite global empirical-risk minimum, the retained certificate expression underestimates a loss increase by over $210\times$. We derive a globally valid, gauge-invariant feature-space repair. Second, common-row softmax shifts preserve predictions and the exact contraction while making the proxy unbounded. Even standard reference-class choices double it on average relative to the centered representation. For a single fixed-feature example with at least three classes, scalar retuning generically cannot align the induced probability updates. Row centering gives the orbit-minimized bound and restores value and full-model gradient invariance under this symmetry. Across 45 paired one-step tests on algorithmic and image models, amplified shifts separate raw-regularized predictors while quotient-regularized predictors remain aligned. Long-horizon CIFAR-10 experiments show substantial, reversible suppression of generalization, while evidence for selective delay after memorization is less consistent. Together, these results show that validity as a curvature upper bound does not by itself justify either inversion into a robustness certificate or differentiation into an intrinsic training intervention.
\end{abstract}

\section{Introduction}

Flatness connects local loss geometry to stability: near a stationary solution, curvature controls the leading response to small parameter perturbations. This motivates explanations of generalization, sharpness-aware training, and robustness guarantees \citep{tsuzuku2020normalized,petzka2021relative,foret2021sharpness}. Using flatness as a regularizer or certificate, however, requires more than a valid curvature upper bound.

\emph{Relative} flatness was proposed to address one important form of coordinate dependence by contracting curvature with layer scale \citep{petzka2021relative}. In practice, recent work has operationalized it through a tractable upper bound on the last-layer contraction. This same proxy underlies the two uses introduced above: it was inverted into an adversarial-robustness certificate \citep{walter2026flatness} and differentiated as a regularizer \citep{han2025flatness} to test whether flatness causally affects generalization during grokking, a setting that temporally separates memorization from generalization \citep{power2022grokking}.

The difficulty is that flatness measures can change under reparameterizations that leave the predictor fixed \citep{dinh2017sharp}. For operational uses, identification and soundness are distinct: a parameterization-dependent certificate can still be valid, whereas a representation-independent training intervention requires the same functional update across equivalent representatives. We therefore distinguish gauge dependence from invalid certification, and parameter-specific training effects from interventions on intrinsic predictor properties.

Softmax classifiers have a simple function-preserving symmetry: adding the same vector to every classifier row introduces a common logit offset and leaves the represented function unchanged. The exact relative-flatness contraction respects this symmetry, but the operational proxy does not. As a result, a certificate derived from the proxy depends on parameterization, while regularizing it can send function-equivalent models in different functional directions. In this precise sense, the proxy is not a function of the predictor. Independently, the certificate derivation drops a pointwise first-order term that dataset stationarity does not eliminate.

The key point is that approximation validity does not automatically justify downstream use. The same valid upper bound is inverted into a robustness radius and differentiated into a training objective, exposing distinct failure modes.

Our main contributions are:

\begin{itemize}
    \item \textbf{Certification.} We show that empirical-risk stationarity does not remove pointwise first-order loss terms. A finite global-minimum counterexample invalidates the resulting curvature-only certificate by over $210\times$, and we derive a globally valid, gauge-invariant feature-space bound.

    \item \textbf{Intervention.} We show that the raw proxy is unbounded within a softmax function-equivalence class even though predictions and the exact contraction are unchanged. Row centering yields its orbit-minimized quotient, and scalar retuning generically cannot align the induced functional updates for a fixed-feature multiclass example with $K \geq 3$.

    \item \textbf{Empirical evidence.} Across 45 paired one-step tests on algorithmic and image models, function-equivalent parameterizations produce different updates under the raw proxy while quotient-based interventions remain aligned. Long-horizon CIFAR-10 experiments show substantial suppression and recovery, with less consistent evidence for selective delay after memorization.
\end{itemize}

\section{Preliminaries}
\label{sec:claims}

\subsection{Two operational uses of the proxy}

We study two downstream uses of the same relative-flatness proxy. \citet{walter2026flatness} invert it into a robustness radius, where validity requires a pointwise loss bound. \citet{han2025flatness} differentiate it as a regularizer, where an intrinsic interpretation requires equivalent predictors to receive the same functional update. These uses expose distinct requirements that are not implied by the validity of the underlying curvature upper bound.

\subsection{Relative flatness and functional identification}

\paragraph{The operational relative-flatness proxy.}

Consider a $K$-class model with feature $h\in\R^d$, classifier $W\in\R^{K\times d}$, logits $z=Wh$, and probabilities $p=\softmax(z)$. The matrix $P(p)=\operatorname{Diag}(p)-pp^\top$ is both the covariance of the one-hot encoding of $Y\sim\operatorname{Cat}(p)$ and the softmax Jacobian $\partial p/\partial z$. For one cross-entropy example, the classifier Hessian is $P(p)\otimes hh^\top$. The last-layer relative-flatness contraction is
\begin{equation}
	\kexact(W,h)=\lVert h\rVert_2^2\Tr\!\left(W^\top P(p)W\right).
	\label{eq:exact}
\end{equation}
Because $P(p)$ is positive semidefinite, a trace inequality gives
\begin{align}
	\kexact(W,h)
	 & \leq \lVert W\rVert_F^2\lVert h\rVert_2^2\Tr P(p)\nonumber \\
	 & =\underbrace{\lVert W\rVert_F^2\lVert h\rVert_2^2
		\left(1-\lVert p\rVert_2^2\right)}_{\Braw(W,h)}.
	\label{eq:raw}
\end{align}
Equation~\eqref{eq:raw} is the operational proxy studied below.

\paragraph{An exact softmax symmetry.}

Before interpreting $\Braw$ as a property of the predictor, consider adding the same vector $v\in\R^d$ to every row of $W$:
\begin{equation}
	G_v(W)=W+\one v^\top.
	\label{eq:gauge}
\end{equation}
Since $G_v(W)h=Wh+(v^\top h)\one$, softmax invariance to common logit shifts implies that $W$ and $G_v(W)$ define the same predictor. Moreover, $P(p)\one=0$ makes $\kexact$ invariant, whereas $\Braw$ generally changes through $\lVert W\rVert_F^2$.

The redundant common-row component can be removed explicitly. Let
\begin{equation}
	\bar w=K^{-1}W^\top\one,
	\qquad
	\Pi=I_K-K^{-1}\one\one^\top,
	\qquad
	W_c=\Pi W=W-\one\bar w^\top.
	\label{eq:centered-classifier}
\end{equation}
Each row of $W_c$ is its original row minus the mean row $\bar w^\top$. The centered classifier is unchanged by $G_v$. Evaluating the same upper-bound expression after this centering defines the quotient proxy
\begin{equation}
	\Bquot(W,h)=\lVert W_c\rVert_F^2\lVert h\rVert_2^2
	\left(1-\lVert p\rVert_2^2\right).
	\label{eq:quotient}
\end{equation}

\paragraph{Three distinct requirements.}

A candidate score $S(W,h)$ can be assessed along three distinct axes:
\begin{enumerate}
	\item \emph{Approximation validity:} if $S$ is used as an upper-bound proxy, does $\kexact(W,h)\leq S(W,h)$ hold?
	\item \emph{Value identification:} do function-equivalent parameters receive the same score, $S(G_v(W),h)=S(W,h)$?
	\item \emph{Interventional identification:} for model output $F(\theta)$, is the Euclidean gradient-flow velocity $\mathcal V_S(\theta)=J_F(\theta)\nabla_\theta S(\theta)$ constant along each orbit generated by $G_v$? This tests changes in predictions rather than score values.
\end{enumerate}
Throughout, identification refers only to the common-row symmetry. For this affine translation, differentiable value invariance implies invariance of the Euclidean gradient field. Section~\ref{sec:proxy} applies this checklist to $\Braw$, $\Bquot$, and $\kexact$.

\subsection{Relation to prior work}

Parameter-space sharpness can change under function-preserving reparameterizations, and several flatness measures address different forms of coordinate dependence \citep{dinh2017sharp,tsuzuku2020normalized,petzka2021relative,jang2022reparam}. Quotienting symmetries is also established \citep{pittorino2022toroids}, while common-shift non-identifiability and its standard identification conventions are classical in multinomial-logit models \citep{hastie2009elements,zahid2009ridge}. Our contribution connects this redundancy to both certification and training, proves that scalar recalibration fails almost everywhere in the fixed-feature multiclass case, and derives the orbit-minimized repair.

\section{One symmetry, two operational failures}
\label{sec:proxy}

The common-row symmetry exposes a shared identification defect in both downstream uses. The following result separates this failure from the certificate's distinct pointwise soundness problem.

\begin{proposition}[Gauge failure and quotient repair]
	\label{prop:nonidentification}
	For every $W,h,v$, the transformation $G_v$ preserves the class-probability vector $p=\softmax(Wh)$ and the exact contraction $\kexact$. If $h\neq0$, $v\neq0$, and $p$ is non-degenerate, then $\Braw(G_{tv}(W),h)\to\infty$ as $|t|\to\infty$. The centered classifier $W_c$ is the unique minimum-Frobenius representative of the orbit. Consequently, the quotient proxy is the orbit-minimized raw bound:
	\begin{equation}
		\Bquot(W,h)=\inf_v\Braw(G_v(W),h),\qquad
		\kexact(W,h)\leq\Bquot(W,h)\leq\Braw(W,h).
		\label{eq:ordering}
	\end{equation}
	Finally, both $\Bquot$ and $\kexact$ have gauge-invariant full-model gradients.
\end{proposition}

\begin{proof}
	The transformed logits are $Wh+(v^\top h)\one$, which softmax removes, and $P(p)\one=0$ removes every common-row term from Equation~\eqref{eq:exact}. In contrast,
	\begin{equation}
		\lVert W+t\one v^\top\rVert_F^2
		=\lVert W\rVert_F^2+2t\langle\one^\top W,v\rangle+t^2K\lVert v\rVert_2^2.
		\label{eq:orbit-norm}
	\end{equation}
	Writing $\bar w=K^{-1}W^\top\one$ gives the orthogonal decomposition $W=W_c+\one\bar w^\top$ and $\lVert W\rVert_F^2=\lVert W_c\rVert_F^2+K\lVert\bar w\rVert_2^2$. Thus centering uniquely minimizes the raw norm. Since $P(p)=\Pi P(p)\Pi$, centering preserves $\kexact$ before the same trace bound proves Equation~\eqref{eq:ordering}. Differentiating the invariance identities proves gradient invariance for this affine translation.
\end{proof}

These statements hold termwise for dataset and minibatch averages. The raw score can therefore differ across parameterizations of the same softmax function. The quotient removes exactly that coordinate.

This quotient is the minimal identification repair of the operational bound: it takes the orbit infimum while retaining the original upper-bound form. Replacing the proxy by $\kexact$ is a different intervention that changes identification, tightness, objective value, and gradient direction at once. We use both below to separate these questions rather than claim that $\Bquot$ is preferable to the exact contraction.

\paragraph{The effect is already visible under standard parameterizations.}
\label{sec:why-gauge}

Identification requires invariance across the entire equivalence class, so one counterexample is enough to reject it. How often large displacements occur during ordinary training is a separate empirical question and cannot rescue a certificate or intervention that changes under a transformation leaving the function exactly unchanged.

A standard reference-class convention gives a non-amplified stress test. Starting from $W_c$, choose class $r$ as the reference and set
\begin{equation}
	W^{(r)}=W_c-\one w_{c,r}^\top.
	\label{eq:reference-gauge}
\end{equation}
The $r$th row is now zero, giving a standard reference-category parameterization of the same softmax function.

\begin{corollary}[Standard identification conventions]
	\label{cor:reference-gauge}
	For every nonzero centered classifier and every $h\neq0$,
	\begin{equation}
		\frac{\Braw(W^{(r)},h)}{\Bquot(W_c,h)}
		=1+\frac{K\lVert w_{c,r}\rVert_2^2}{\lVert W_c\rVert_F^2},
		\qquad
		\frac1K\sum_{r=1}^K
		\frac{\Braw(W^{(r)},h)}{\Bquot(W_c,h)}=2.
		\label{eq:reference-ratio}
	\end{equation}
\end{corollary}

\begin{proof}
	The centered and common-row components in Equation~\eqref{eq:reference-gauge} are orthogonal, so $\lVert W^{(r)}\rVert_F^2=\lVert W_c\rVert_F^2+ K\lVert w_{c,r}\rVert_2^2$. Averaging the second term over rows contributes exactly one additional copy of $\lVert W_c\rVert_F^2$.
\end{proof}

Thus the failure is not confined to arbitrarily amplified gauge shifts: ordinary reference-class conventions change the proxy by a factor of two on average while leaving the predictor unchanged.

\section{From flatness bounds to robustness certificates}
\label{sec:robustness}

Gauge dependence makes the reported radius representation-dependent, but does not by itself establish unsoundness. The soundness failure comes from a separate pointwise first-order omission. Independently, Equation 5 of \citet{walter2026flatness} includes the pointwise term $\Delta\lVert W\rVert_F\lVert\nabla_W\ell_x(W)\rVert_F$. Proposition 8 can omit it at a pointwise minimum, although softmax cross-entropy has no such finite minimizer when $\lVert\phi(x)\rVert_2>0$. Proposition 9, however, assumes only an empirical-risk minimum and applies Proposition 8 to each example. This is invalid because $\nabla_W R_S(W)=0\not\Rightarrow\nabla_W\ell_x(W)=0$ for every $x\in S$.

\begin{proposition}[First-order obstruction]
	\label{prop:first-order-obstruction}
	Let $\ell$ be differentiable at $h$ with $g_h=\nabla_h\ell(h)\neq0$. For $d_\rho=\rho g_h/\lVert g_h\rVert_2$,
	\begin{equation}
		\ell(h+d_\rho)-\ell(h)=\rho\lVert g_h\rVert_2+o(\rho).
	\end{equation}
	Hence no nonnegative $q(\rho)=o(\rho)$, including any fixed quadratic-plus-cubic expression, upper-bounds every pointwise loss increase near $h$.
\end{proposition}

\begin{proof}
	Differentiability gives $\ell(h+d)=\ell(h)+g_h^\top d+o(\lVert d\rVert_2)$. Substitution of $d_\rho$ yields the display. Its positive linear coefficient eventually exceeds every $q(\rho)=o(\rho)$.
\end{proof}

\begin{proposition}[Finite dataset-minimum counterexample]
	\label{prop:robustness-counterexample}
	Let $K=2$, $d=1$, and $\phi(x)=x$, with the rows of $W$ representing classes $0$ and $1$. On the training multiset $S=[(1,1),(1,1),(1,0)]$, $W_\star=(0,\log2)^\top$ is a global empirical-risk minimizer, unique up to common-row shifts. For $y=1$ at $x=1$ and the perturbation $1\mapsto0.99$, the pointwise loss increase is $2.32\times10^{-3}$ while the published quadratic-plus-cubic expression is $1.10\times10^{-5}$: a violation by more than $210\times$.
\end{proposition}

\begin{proof}
	Binary softmax depends on $W=(w_0,w_1)^\top$ only through $\beta=w_1-w_0$, so
	\begin{equation}
		R(\beta)=\tfrac23\log(1+e^{-\beta})+
		\tfrac13\log(1+e^\beta),
	\end{equation}
	Since $R'(\beta)=\softmax(0,\beta)_1-2/3$ is strictly increasing, its unique root is $\beta_\star=\log2$. Hence all minimizers are $(c,c+\log2)^\top$, and $W_\star$ chooses $c=0$. Here $p_1=2/3$, so the two class-1 gradients cancel the class-0 gradient in the empirical average although none vanishes. Take $\mathcal X=[0.99,1]$, $\delta=0.01$, $K=2$, $m=d=1$, and $r=0.99$. These choices satisfy $\phi$ $1$-Lipschitz, $|x|\leq1$, $|\phi(x)|\geq r$, and $\ell_1(1)=\log(3/2)<\log2$. For the class-1 example,
	\begin{align}
		\ell_1(0.99)-\ell_1(1)
		 & =\log(1+2^{-0.99})-\log(1+2^{-1})
		\approx2.32\times10^{-3},            \\
		\frac{\delta^2}{2r^2}\Braw(W_\star,1)+\frac{\delta^3}{24r^3}Km
		 & \approx1.10\times10^{-5}.
	\end{align}
	Their ratio exceeds $210$.
\end{proof}

\begin{corollary}[False certified radius]
	\label{cor:false-radius}
	For loss tolerance $\epsilon=10^{-4}$, the published expression gives the certified radius $\delta_{\rm pub}=2.99\times10^{-2}$, yet the perturbation in Proposition~\ref{prop:robustness-counterexample} has norm $0.01<\delta_{\rm pub}$ and increases the loss by $2.32\times10^{-3}>\epsilon$, so it lies inside the certified ball while violating its guarantee. Since $\phi(x)=x$ and $p_y=2/3>1/2$, this is an input-space counterexample satisfying the stated high-confidence condition.
\end{corollary}

\begin{proposition}[A sound, gauge-invariant feature-space repair]
	\label{prop:functional-robustness}
	The two failures have a common repair: retain the first-order term and replace $\Braw$ with gauge-invariant row-difference geometry. For target $y$, let $g_h=\nabla_h\ell_y(h)=W^\top(p-e_y)$ and define $D_W=\max_{i,j}\lVert w_i-w_j\rVert_2$. Then every feature perturbation $d$ satisfies
	\begin{equation}
		\ell_y(h+d)-\ell_y(h)
		\leq \lVert d\rVert_2\lVert g_h\rVert_2
		+\frac{D_W^2}{8}\lVert d\rVert_2^2.
		\label{eq:functional-robustness-bound}
	\end{equation}
	The uniform Hessian constant $D_W^2/4$ is sharp given only $D_W$. Consequently, for loss tolerance $\epsilon>0$ and $D_W>0$, a certified feature radius is
	\begin{equation}
		\rho_{\rm cert}=
		\frac{2\epsilon}{\lVert g_h\rVert_2+
			\sqrt{\lVert g_h\rVert_2^2+D_W^2\epsilon/2}}.
		\label{eq:functional-certified-radius}
	\end{equation}
	If $D_W=0$, take $\rho_{\rm cert}=\infty$. If the feature map is $L$-Lipschitz with $L>0$, the corresponding input-space radius is $\rho_{\rm cert}/L$.
\end{proposition}

\begin{proof}
	For $t\in[0,1]$, let $q_t=\softmax(W(h+td))$. The feature Hessian at $h+td$ is $W^\top P(q_t)W$, and for every unit $u$,
	\begin{equation}
		u^\top W^\top P(q_t)Wu=\operatorname{Var}_{C\sim q_t}(u^\top w_C)
		\leq\tfrac14(\max_i u^\top w_i-\min_j u^\top w_j)^2\leq D_W^2/4.
	\end{equation}
	The first inequality bounds a scalar variance by one quarter of its squared range. Hence the Hessian operator norm is at most $D_W^2/4$ along the segment. Taylor's theorem gives $\ell_y(h+d)-\ell_y(h)\leq g_h^\top d+(D_W^2/8)\lVert d\rVert_2^2$, and Cauchy--Schwarz yields Equation~\eqref{eq:functional-robustness-bound}. A balanced binary softmax whose rows differ by $D_Wu$ attains the Hessian bound, so the constant is sharp. Setting the right-hand side of Equation~\eqref{eq:functional-robustness-bound} equal to $\epsilon$ and solving for $\lVert d\rVert_2$ gives Equation~\eqref{eq:functional-certified-radius}. Finally, under $G_v(W)=W+\one v^\top$, both $p$ and $D_W$ are unchanged, while $G_v(W)^\top(p-e_y)=g_h+v\one^\top(p-e_y)=g_h$.
\end{proof}

\paragraph{The first-order obstruction persists in learned representations. }

Equation~\eqref{eq:functional-robustness-bound} is a global certificate for arbitrary feature-space perturbations and needs no stationarity, feature lower bound, bounded input, or third-order remainder. It becomes an input-space certificate only with an independently justified Lipschitz bound for the feature map. We therefore evaluate it directly in feature space, asking whether the omitted term is material on learned representations. For each of five trained CIFAR-10 ResNets and a common set of 2,048 test images, we perturb the penultimate feature along $u=g_h/\lVert g_h\rVert_2$ by $d=\Delta\lVert h\rVert_2u$.

\begin{table}[H]
	\caption{\textbf{First-order obstruction on learned representations.} Results pooled over 10,240 model--image pairs. Actual $>$ raw is the fraction exceeding the retained quadratic. Linear/raw measures the omitted term relative to that quadratic, and bound/actual measures the tightness of Equation~(14), with 1 denoting equality. Ratios are medians.}
	\label{tab:e034}
	\centering
	\small
	\input{tables/robustness.tex}
\end{table}

The learned-feature results match the analysis. Through $\Delta=0.003$, the realized loss increase exceeds the retained raw quadratic for every model--image pair. At $\Delta=0.001$, the median linear/raw ratio is $10.60$, following the predicted $1/\Delta$ scaling. The repaired bound covers all 71,680 responses and is tight locally, with median and 90th-percentile bound/actual ratios of $1.039$ and $1.093$ at $\Delta=10^{-5}$.

At tolerance $0.10$, the radius from raw curvature alone admits an above-tolerance loss increase for $24.43\%$ of pairs. A matched common-row shift multiplies the raw quadratic by $100\times$, while probabilities and loss responses change by at most $3.8\times10^{-14}$. The analytic counterexample supplies the logical refutation, while these learned-feature results show that both mechanisms remain material on trained models.

\section{From flatness proxies to training interventions}
\label{sec:causal-use}

For an intrinsic interpretation, a proxy-based training intervention should assign the same functional update to common-row-equivalent models. The next proposition shows that $\Braw$ fails this requirement.

\begin{proposition}[Functional-update failure]
	\label{prop:first-order-failure}
	Fix $h$, write $s(W,h)=\lVert h\rVert_2^2(1-\lVert p\rVert_2^2)$, let $W'=G_v(W)$, and define $\Delta_v(W)=\lVert W'\rVert_F^2-\lVert W\rVert_F^2$. Then
	\begin{equation}
		\nabla_W\Braw(W',h)-\nabla_W\Braw(W,h)
		=2s\one v^\top+\Delta_v(W)\nabla_Ws(W,h).
		\label{eq:raw-gradient-difference}
	\end{equation}
	The first term is function-null, but the remaining probability velocity is
	\begin{equation}
		J_Wp\,\mathrm{vec}(\nabla_W\Braw(W',h)-\nabla_W\Braw(W,h))
		=-2\Delta_v(W)\lVert h\rVert_2^4P(p)^2p.
		\label{eq:raw-functional-difference}
	\end{equation}
	For finite logits and $h\neq0$, the raw intervention therefore fails exactly when $\Delta_v(W)\neq0$ and $p$ is not uniform.
\end{proposition}

\begin{proof}
	Gauge invariance gives the same $s$, $\nabla_Ws$, and probability Jacobian at $W$ and $W'$. Since $\Braw=s\lVert W\rVert_F^2$, differentiation gives $\nabla_W\Braw=2sW+\lVert W\rVert_F^2\nabla_Ws$. Subtracting proves Equation~\eqref{eq:raw-gradient-difference}. The common-row term adds the same amount to every logit and is function-null. Finally, $\nabla_Ws=-2\lVert h\rVert_2^2P(p)p\,h^\top$, which yields Equation~\eqref{eq:raw-functional-difference}. For finite logits, $\ker P(p)=\operatorname{span}\{\one\}$.
\end{proof}

\begin{corollary}[Scalar retuning cannot generically restore identification]
	\label{cor:no-retuning}
	Let $F(\theta)$ be a differentiable model output with parameters $\theta=(W,\xi)$ and the common-row symmetry above. For any differentiable gauge-invariant $a(\theta)$, set $S(\theta)=a(\theta)\lVert W\rVert_F^2$ and $c=\lVert W\rVert_F^2$. Along one orbit its functional velocity has the form
	\begin{equation}
		\mathcal V_S=A+cB,\qquad
		A=2aJ_{F,W}\operatorname{vec}(W),\quad B=J_F\nabla_\theta a,
		\label{eq:affine-functional-velocity}
	\end{equation}
	where $A$ and $B$ are orbit-invariant. If they are linearly independent, then for representatives with $c_1\neq c_2$ no nonzero scalars $\lambda_1,\lambda_2$ satisfy $\lambda_1\mathcal V_S(\theta_1)=\lambda_2\mathcal V_S(\theta_2)$. For $S=\Braw$, $F=p$, and a single fixed-feature example with $h\neq0$, this nondegeneracy holds for $(K-1)$-dimensional Lebesgue-almost-every $p$ in the simplex interior whenever $K\geq3$.
\end{corollary}

\begin{proof}
	Because $F$ and $a$ are gauge-invariant, $J_F$, $a$, and $\nabla_\theta a$ are the same at every representative. The product rule gives $J_F\nabla_\theta S=2aJ_{F,W}\operatorname{vec}(W)+cJ_F\nabla_\theta a$. The gauge shift changes $W$ by $\one v^\top$, but $J_{F,W}\operatorname{vec}(\one v^\top)=0$. Thus both $A$ and $B$ are orbit-invariant, proving Equation~\eqref{eq:affine-functional-velocity}.

	Suppose scalar retuning aligned two representatives. Expanding the equality of their velocities gives $(\lambda_1-\lambda_2)A+(\lambda_1c_1-\lambda_2c_2)B=0$. Linear independence forces both coefficients to vanish. Hence $\lambda_1=\lambda_2$, and their nonzero common value then forces $c_1=c_2$, a contradiction. Any shared gauge-invariant primary-loss field is identical at both representatives and cancels before this comparison.

	For the raw fixed-feature case, take $F=p$ and write $z=Wh$. Since $p=\softmax(z)$, we have $z=\log p+\gamma\one$ for some scalar $\gamma$. Using $P(p)\one=0$ gives $J_{p,W}\operatorname{vec}(W)=P(p)z=P(p)\log p$. Also, $J_{p,W}\operatorname{vec}(\nabla_Ws)=-2\lVert h\rVert_2^4P(p)^2p$. Therefore $A$ and $B$ point, up to nonzero scalar factors, along $P(p)\log p$ and $P(p)^2p$. For $K=3$, their first-two-coordinate minor at $p=(1/2,3/10,1/5)$ equals $\frac{9}{25000}\log(640/6561)\neq0$. This minor is analytic in the simplex interior, so its zero set has measure zero. The set where all minors vanish is a subset of that zero set. For $K>3$, assigning sufficiently small positive mass to the additional classes preserves a nonzero minor and gives the same conclusion.
\end{proof}

The raw objective also rewards motion in the function-null common-row direction. Let $\bar w=K^{-1}W^\top\one$. Because the primary loss and $s$ are invariant to common-row shifts, their gradients have zero row mean. Taking the row mean of a gradient-ascent step on $\alpha\Braw$ therefore gives
\begin{equation}
	\bar w_{t+1}=(1+2\eta\alpha s_t)\bar w_t,
	\label{eq:mean-row-dynamics}
\end{equation}
For $\alpha>0$, $s_t>0$, and $\bar w_t\neq0$, repeated steps amplify this component even though it leaves probabilities unchanged. This is the gauge-runaway prediction tested below. Quotient and exact objectives do not contain this function-null growth term.

\paragraph{Empirical setup.}
\label{sec:experiments}

We first ask whether the defect changes realistic one-step updates, then examine how it manifests over full training trajectories. In every pair, non-classifier parameters, optimizer state, minibatches, and initial probabilities match. We calibrate the auxiliary coefficient once at the centered model and reuse it unchanged after the gauge shift. Tests cover addition and multiplication Transformers, a subtraction MLP, frozen ResNet-18 features, and end-to-end CIFAR-10 ResNet-18 checkpoints. Long-horizon experiments use fixed decision rules specified before evaluation. We draw no trajectory claim from one step.

\subsection{Standard and naturally stored gauges}
\label{sec:standard-gauges}

Any equivalent reparameterization is enough to break identification, but ordinary gauge scales determine practical relevance. In ten frozen-feature CIFAR-10 probes evaluated as stored, centering $W$ as in $\Bquot$ removes a mean-row component containing $1.93$--$2.36\%$ of $\lVert W\rVert_F^2$. In five end-to-end ResNets, no shift is applied. The raw/quotient ratio starts at $1.103$--$1.118$, and weight decay drives it toward the centered value of one (Figure~\ref{fig:standard-gauge}a).

\begin{figure}[!htbp]
	\centering
	\includegraphics[width=\linewidth]{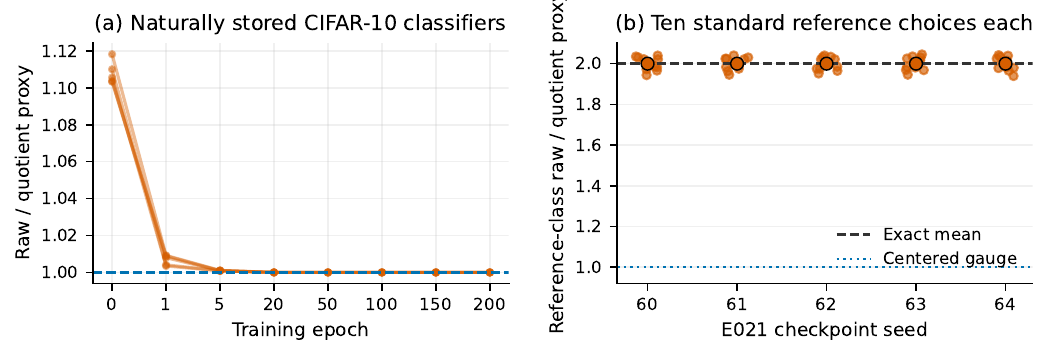}
	\caption{\textbf{Gauge scale without amplified shifts.} \textbf{(a)} With no gauge shift, weight decay drives five ordinary end-to-end classifiers toward the centered ratio of one. \textbf{(b)} Each checkpoint is reparameterized ten times by setting one class row to zero. Black marks the exact mean of two from Corollary~\ref{cor:reference-gauge}.}
	\label{fig:standard-gauge}
\end{figure}

For each of five centered checkpoints, we subtract each class row from all ten rows in turn. The chosen row becomes zero, but probabilities do not change. The 50 resulting reference-class parameterizations have raw/quotient ratios of $1.939$--$2.045$ and an exact within-checkpoint mean of two (Figure~\ref{fig:standard-gauge}b). Thus a standard equivalent representation can nearly double the proxy even when natural storage is nearly centered. Corollary~\ref{cor:no-retuning} independently rules out generic repair by gauge-specific scalar retuning.

\subsection{Equivalent predictors receive different updates}

On ten trained modular-addition Transformer checkpoints, common-row shifts make $\Braw$ $1$, $10$, or $100$ times its centered value while probabilities, $\Bquot$, and $\kexact$ stay fixed (Figure~\ref{fig:gauge}a). At $100\times$, the raw full gradient differs from the centered one by approximately $99$ times the centered-gradient norm, despite probability differences below $5.11\times10^{-15}$. Quotient and exact gradient differences remain below $6.23\times10^{-12}$ (Figure~\ref{fig:gauge}b).

\begin{figure}[!htbp]
	\centering
	\includegraphics[width=\linewidth]{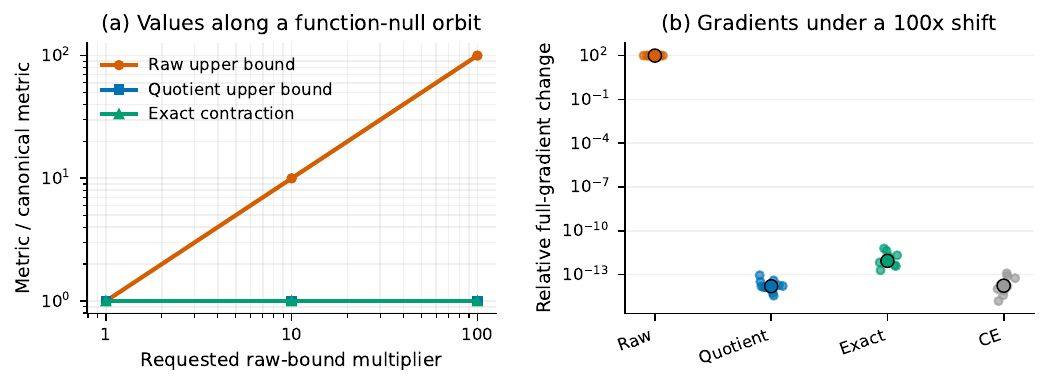}
	\caption{\textbf{Only raw changes along a function-null orbit.} Ten checkpoints (thin) and medians (emphasized). Probabilities, $\Bquot$, and $\kexact$ stay fixed.}
	\label{fig:gauge}
\end{figure}

The central empirical question is whether this symmetry defect changes the function after one matched optimizer step. Across 45 centered/100$\times$-shifted pairs, it does. We test ten each for addition Transformers, multiplication Transformers, subtraction MLPs, and frozen-feature CIFAR-10, plus five end-to-end ResNets. Initial paired probability differences are below $10^{-10}$. After one update, raw separates every pair, while quotient, available exact, and CE controls remain aligned near numerical precision (Figure~\ref{fig:local}).

\begin{figure}[!htbp]
	\centering
	\includegraphics[width=0.85\linewidth]{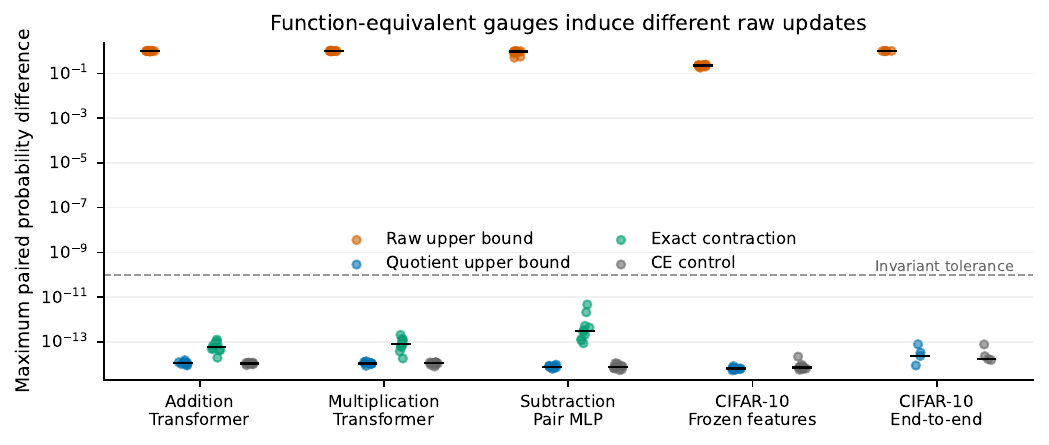}
	\caption{\textbf{One-update functional deviation.} Raw separates equivalent functions. Quotient, available exact, and CE controls remain aligned near machine precision.}
	\label{fig:local}
\end{figure}

The separation is large beyond the algorithmic models. Subtraction MLPs reach maximum paired probability differences of $0.504$--$0.984$ with $14.8$--$21.3\%$ prediction disagreement. Raw differences are $0.184$--$0.242$ in all ten frozen-feature CIFAR-10 probes and $0.990$--$1.000$ in all five end-to-end checkpoints, while the largest quotient difference across image settings is $7.87\times10^{-14}$. The effect is therefore not confined to a toy model or a frozen representation.

Quotienting repairs gauge identification, but $\Bquot$ is not tight. At the centered modular checkpoints, its value is $102.95$--$108.44$ times $\kexact$. A separate five-seed temporal null changes only the function-null coordinate used to report $\Braw$, keeping it about $100\times$ its paired control value while probabilities, centered classifiers, non-classifier parameters, optimizer states, and minibatches match. All five pairs have identical grokking onset. Thus the diagnostic value alone cannot determine generalization.

\FloatBarrier

\subsection{Long-horizon training dynamics}

Separate long-horizon CIFAR-10 runs show suppression and recovery. Before removal, raw-regularized models trail the paired CE control by $10.61$ percentage points on average (range $7.04$--$16.58$). Accuracy rises by $11.90$ points ($8.48$--$17.11$) after regularizer removal and the scheduled optimizer changes. This supports suppression and recovery, but not a consistent selective-delay pattern. Three runs memorize without reaching the required ten-point test gap, while two reach the gap without memorizing. The unchanged released loop likewise shows suppression and recovery in all three published seeds, although only one satisfies the same composite endpoint.

Modular addition is less decisive. No observed raw run shows the required 3000-step onset separation, and three raw onsets remain unobserved by step 100000. The exact-gradient-matched diagnostic does not resolve whether exact flatness is necessary. Quotient variants satisfy the composite CIFAR-10 endpoint in 1/5 and 2/5 runs, so they likewise do not resolve the stronger claim.

With the raw objective and a quotient cap that leaves the common-row component unconstrained, all five CIFAR-10 stability runs cross the runaway threshold at epoch 4 and later exceed common-row norm $10^3$.

\FloatBarrier
\section{Limitations and conclusion}

The quotient repair removes only final-softmax common-row symmetry and remains about two orders above $\kexact$. Natural raw/quotient gaps are small at final CIFAR-10 checkpoints, although standard reference-class conventions change $\Braw$ about twofold. The $100\times$ tests amplify signal deliberately. The learned-feature study is not an input-space benchmark, and long-horizon tests cover modular addition and CIFAR-10.

Within this scope, the raw proxy fails in two distinct ways. In robustness certification, empirical-risk stationarity does not remove the pointwise first-order loss term; retaining it yields a globally valid, gauge-invariant feature-space certificate. As a training objective, the raw proxy assigns different functional updates to equivalent predictors; quotienting the common-row symmetry restores identification. Long-horizon experiments show strong reversible suppression on CIFAR-10, but do not consistently isolate selective delay after memorization, and therefore do not resolve whether exact relative flatness is necessary for generalization. More broadly, validity as a geometric upper bound is not enough for downstream use. Certification requires a sound pointwise bound, while intrinsic training interventions require symmetry-consistent functional updates.\footnote{Generative AI tools assisted with experimental planning, code development, manuscript preparation, and checks of mathematical derivations and experimental results.}

\label{main-text-end}
\bibliography{references}
\bibliographystyle{plainnat}

\end{document}

%% file: tables/robustness.tex
\begin{tabular}{@{}rrrr@{}}
\toprule
$\Delta$ & Actual $>$ raw (\%) & Linear/raw & Bound/actual \\
\midrule
$10^{-5}$ & 100.00 & 1060.27 & 1.039 \\
$10^{-4}$ & 100.00 & 106.03 & 1.393 \\
$10^{-3}$ & 100.00 & 10.60 & 4.908 \\
$3\!\times\!10^{-3}$ & 100.00 & 3.53 & 12.581 \\
$10^{-2}$ & 98.34 & 1.06 & 38.012 \\
\bottomrule
\end{tabular}